\documentclass[letterpaper, 10 pt, conference]{ieeeconf}  

\IEEEoverridecommandlockouts                              

\usepackage{graphics} 
\usepackage{epsfig} 
\usepackage{mathptmx} 
\usepackage{times} 
\usepackage{amsmath} 
\usepackage{amssymb}  
\usepackage{algpseudocode}
\usepackage{algorithm}

\newtheorem{definition}{Definition}
\newtheorem{theorem}{Theorem}
\newtheorem{proposition}{Proposition}
\newtheorem{lemma}{Lemma}
\newtheorem{corollary}{Corollary}

\newtheorem{assumption}{Assumption}

\title{\LARGE \bf
Correct-by-Construction Navigation Functions\\ with Application to Sensor Based Robot Navigation
}

\author{Savvas G. Loizou$^{1}$ and Elon D. Rimon$^{2}$
\thanks{$^{1}$Dept. of Mechanical and Materials Science~\mbox{Engineering,}
        Cyprus University of Technology, Limassol, Cyprus. {\tt\small  savvas.loizou@cut.ac.cy}.}%
\thanks{$^{2}$Dept. of Mechanical Engineering, Technion Israel Institute of Technology, Haifa, Israel. {\tt\small rimon@technion.ac.il}.}%
}

\begin{document}

\maketitle
\thispagestyle{empty}
\pagestyle{empty}

\begin{abstract}
This paper brings together the concepts of navigation transformation and harmonic functions to form navigation functions that are correct-by-construction in the sense that no tuning is required.  The form of the navigation function is explicitly related to the number of obstacles in the environment. This enables application of navigation functions for autonomous robot navigation in partially or fully unknown environments, with the capability of on-the-fly adjustment of the navigation function when new obstacles are discovered by the robot. Appropriate navigation controllers, applicable to robots with local, sector bounded sensing, are presented and analyzed for a~kinematic point-mass robot and then for the dynamic point-mass robot system. The closed form nature of the proposed navigation scheme provides for online, fast-feedback based navigation. In addition to the analytic guarantees, simulation studies are presented to verify the effectiveness of the methodology.
\end{abstract}

\section{Introduction}

\noindent Navigation Functions \cite{KodRimNF90} is an active research topic in the field of robotic navigation. There are numerous successful examples in the literature of Navigation Function applications to navigation problems of increasing complexity \cite{Dimarogonas_Dual_12}. One of the major challenges when setting up a Navigation Functions has to do with its tuning. A construction can be a candidate Navigation Function, but to actually have the Navigation Function properties, an appropriate selection of the $k$ parameter has to be made. Such a $k$ is always guaranteed to exist, however calculation of its lower bound is not straightforward. Several solutions have appeared in the literature to tackle the issue of sensor based reactive planning, see e.g. \cite{arslankoditschek19}, \cite{Vasilopoulos20} however setting up a correct-by-construction Navigation Function on an arbitrary workspace is still an open issue.

Recent results on the Navigation Transformation \cite{NavTranTRO17} and Harmonic Function based Navigation Functions \cite{LoizouCDC11} indicated the feasibility of a tuning controller. However the analysis provided in the current work demonstrates that a tuning controller is not required and that for Harmonic Function based Navigation Functions utilizing the Navigation Transformation, parameter $k$, when explicitly provided as a function of the number of workspace obstacles is sufficient to generate a correct behavior. Kinematic and dynamic controllers were produced to handle environments with only partial workspace knowledge and robot sensing limited in a bounded sensing sector. The contributions of the current work are summarized in the list below:
\begin{itemize}
    \item A correct-by-construction Navigation Function
    \item On-the-fly addition of obstacles in the workspace
    \item A relaxed version of the Morse Property of the Navigation Functions that provides equivalent performance characteristics
    \item Navigation capability with arbitrary, non-zero, sector-bounded local sensing 
    \item Kinematic controller whose vector field does not vanish at the vicinity of saddle points
    \item Dynamic controller with critical damping and bounded maximum velocity
\end{itemize}

The rest of the paper is organized as follows: Section \ref{sec:prelims} presents preliminary notions, sections \ref{sec:HFcontruction} and \ref{sec:NFanalysis} present the construction and analysis of the navigation function, section \ref{sec:CtrlDesign} presents the controller design, section \ref{sec:sims} presents simulation results and section \ref{sec:conclude} concludes the paper.

\section{Preliminaries}
\label{sec:prelims}

\noindent This section introduces navigation functions terminology, the two mobile robot types, the robot sensor model, then the navigation function and controller design problem.

{\bf Basic terminology:} If $K$ is a set, then  
The $n$-dimensional \emph{sphere world} as defined in \cite{KodRimNF90}
is a~compact and connected set in $\mathbb{R}^n$ that is bounded by an~outer  $(n-1)$-dimensional sphere and contains $M$ disjoint internal spheres that represent obstacles. It serves as a~topological model for robot workspaces defined below. 

Given a~smooth function $f(\cdot):\mathbb{R}^n\rightarrow \mathbb{R}^n$, we denote the Jacobian matrix of this function as $J_f(\cdot)$ and the Jacobian determinant by $\left\|J_f(\cdot)\right\|$. Given a~smooth function $\phi:\mathbb{R}^n\rightarrow \mathbb{R}$ we denote the Hessian matrix of $\phi$ as $\mathcal{H}_\phi\left(\cdot\right)$. 


{\bf Robot workspaces:} The navigation functions will be constructed in the following point worlds~\cite{NavTranTRO17}.


\begin{definition}
	Let $P_i \in \mathbb{R}^n$ for $i = 1 \ldots M$ be $M$ discrete point obstacles in $\mathbb{R}^n$. Then a \emph{point world}
	is defined as the set $\mathcal{P} = \mathbb{R}^n - \{ P_1,\ldots,P_M\}$. 
\end{definition}


We will restrict our attention to robot workspaces that are topologically equivalent to sphere worlds as follows.

\begin{definition}
	The \emph{robot workspace} $\mathcal{W}$ is a~compact and connected set in $\mathbb{R}^n$ topologically equivalent to a~sphere world. 
\end{definition}

A valid robot workspace has an~outer boundary denoted ${\cal O}_0$ and disjoint internal obstacles denoted ${\cal O}_i$ for $i=1\ldots M$. 
By the nature of the underlying sphere world, the \emph{interior} of a~robot workspace is diffeomorphic to a~point world~$\mathcal{P}$.  
The mapping of complex robot workspaces to their point-world models will use the following coordinate transformation.

\begin{definition}[\cite{NavTranTRO17}]
	\label{def:NavTran}
	A \emph{navigation transformation} is a~smooth coordinate transformation (a~diffeomorphism) $\Phi:\stackrel{\circ}{\mathcal{W}}\rightarrow \mathcal{P}$, that maps the interior of the robot workspace to a point world. 
\end{definition}


{\bf Robot models:} This paper considers two types of mobile robots that navigate in planar environments populated by obstacles. The first is the trivial kinematic integrator that models a~kinematic point-mass robot:

\begin{equation}\label{eq:holint}
\dot x(t) = u(t)
\end{equation}

\noindent where $x\in \mathbb R^2$ is the robot position and $u(t) \in \mathbb{R}^2$ is a~piecewise continuous control input. 
The more demanding model is the second-order dynamic integrator for a~ point robot of mass~$m$: 

\begin{equation}\label{eq:dynmodel}
m \cdot \ddot x(t) = f(t)
\end{equation}

\noindent where $(x,\dot x)$ is the robot state and  $f(t)\in \mathbb R^2$ is a~piecewise continuous control input representing force applied to the robot.

{\bf Sensor model:} Instead of just considering a~uniform sensing radius for the robot, this paper will use a~more general robot centered \emph{symmetric sensing sector} defined as

\begin{equation}
    \label{eq:sensing}
    S_{r,\theta} (x ,\dot x) =   \left\{(r',\theta') \left|\;\right. 0 \le r' \le r, \;   \left| \theta' \right| \leq  \frac{\theta}{2} \right\} 
\end{equation}

\noindent where the sensing sector's 
pole is at $x$ and the sensing sector's axis is aligned with the robot velocity~$\dot x$.

{\bf Navigation function:} Navigation functions were originally defined as analytic Morse functions~\cite{KodRimNF90}. This paper introduces a~smooth version that relaxes the Morse property as follows. Denote by $x_0$ and $x_d$ the robot initial and destination positions. 

\begin{definition}\label{def:NFDefinition}
	Let $\mathcal{W}
	\subset \mathbb{R}^n$ be a compact connected smooth
	manifold with boundary representing a~valid robot workspace. A map $\varphi:
	\mathcal W \rightarrow [0,1]$, forms a~\emph{navigation function} if it is:
	\begin{enumerate}
		\item Smooth on $\mathcal W$.
		\item Polar on $\mathcal W$,  $\varphi$ has a~unique minimum at $x_d
		\in \stackrel{\circ}{\mathcal{W}}$ such that  $\varphi(x_d)=0$.
		\item The critical points of $\varphi$ except $x_d$ are isolated saddle points 
		with basins of attraction of zero measure.
		\item Admissible on $\mathcal W$, $\varphi$ attains maximal value of unity on $\partial\mathcal W$.
	\end{enumerate}
\end{definition}

{\bf Problem Description:} In the problem considered in this paper the workspace is assumed to contain static obstacles. Each obstacle has a~neighborhood that is disjoint from the neighborhoods of the other obstacles. In each neighborhood a~local \emph{navigation transformation} is provided that maps the obstacle neighborhood to a~point-obstacle neighborhood in the corresponding point world. However, the obstacle locations are \emph{not} known to the robot. As the robot navigates to the destination it detects new obstacles. 
The position and shape of the obstacles are determined at \emph{discrete time instants} and accumulated by the robot during navigation. We can now state the problem considered in this paper.

{\bf Problem statement:} Construct correct-by-construction navigation functions and provide the corresponding feedback control law so that the robot systems \eqref{eq:holint} and \eqref{eq:dynmodel}
can navigate from any initial position to a~specified  destination avoiding collisions with obstacles in the environment.

Note that navigation functions guarantee arrival to the destination from any initial position except~for~a~set~of~zero~measure, which is unstable in the sense that small disturbances will set the robot on a path that reaches the destination.

\section{Construction of Navigation Function based on Harmonic Potentials}
\label{sec:HFcontruction}

\noindent Assume a valid robot workspace $\mathcal{W}$ and a~navigation transformation $\Phi: \mathcal{W} \rightarrow \mathcal{P}$ having the additional property that  $\lim\limits_{x \rightarrow \mathcal{O}_0}   \left\| \Phi (x) \right\| =\infty$, where $\mathcal{O}_0$ is the workspace outer boundary (see \cite{LoizouCDC11} for such an instance). A \emph{point obstacle harmonic potential} is defined as 

\[ \phi_i(h) = \ln \left(\left\|h-P_i \right\|^2 \right),
\] 

\noindent and the \emph{destination harmonic potential} is defined as

\[ 
\phi_d(h) =  \ln \left(\left\|h-P_d \right\|^2 \right)\] 
where $P_d = \Phi(x_d)$.  Note that $-\phi_i(P_i)= +\infty$ while $\phi_d(P_d)= -\infty$.
The \emph{harmonic point world potential function,} $\phi_k: \mathcal{P} \rightarrow R$, is defined as

\[ 
\phi_k (h) =  \phi_d (h) - \frac{\mbox{\small $1$}}{k}\sum \limits_{i=1}^{M} \phi_i (h) 
\] 

\noindent where $k$ a positive parameter and $M$ is the total number of internal obstacles in $\mathcal{W}$. The function $\phi_k$ attains values in the extended real line while navigation functions are required to attain values in $[0,1]$.
The mapping to the unit interval is achieved by composition with the function

\begin{equation}
\label{eq:switch}
\sigma\left(x\right) \triangleq \frac{e^x}{1+e^x} \,.
\end{equation}

\noindent that maps the extended real line to the unit interval $[0,1]$, such that 
$\sigma(-\infty)=0$ and $\sigma(+\infty)=1$.
The candidate \emph{workspace navigation function,} $\varphi_k: \mathcal{W} \rightarrow R$, is
obtained by composition with the workspace to point-world coordinate transformation:

\begin{equation}
\label{eq:NFcand}
\varphi_k (x) = \sigma \circ \phi_k \circ \Phi (x).
\end{equation}

\section{Navigation Function Analysis}
\label{sec:NFanalysis}

\noindent This section builds up several properties of the point world
harmonic navigation function, $\phi_k(h)$, that will lead to the conclusion that $\varphi_k (x) = \sigma \circ \phi_k \circ \Phi (x)$ is a~navigation function. We start with the following result.

\begin{lemma}
\label{lemma:locminfree}
The point world harmonic potential $\phi_k: \mathcal{P} \rightarrow R$ is free of any local minima.
\end{lemma}

\begin{proof}
Since $\phi_k: \mathcal{P} \rightarrow R$ is a superposition of the harmonic potentials $\phi_i(\cdot)$, $i\in \left\{1,\ldots M \right\}$ and $\phi_d(\cdot)$, it will also be harmonic, which implies by definition that \[\Delta \phi_k(h) \triangleq 0.\] Since  $\Delta \phi_k(h) = \mathrm{tr}(\mathcal{H}_{\phi_k}(h))$, we will have that $\mathrm{tr}(\mathcal{H}_{\phi_k}(h))=0$. The trace of a matrix is the first principal invariant of a matrix, hence $I_1 = \mathrm{tr}(\mathcal{H}_{\phi_k}(h)) = \lambda_1+\lambda_2 =0$  This implies that at the  critical points we cannot have both eigenvalues positive, hence a local minimum is not possible.
\end{proof}

Next consider the \emph{attractivity} of $\phi_k$, i.e. that the negated gradient trajectories of $\phi_k$ do not escape to infinity in~$\mathcal{P}$.


\begin{lemma}
	\label{lem:attactive}
	For $k > M$ we have that  $\lim \limits_{\left\| h \right\| \rightarrow \infty } \phi_k (h) = + \infty$. 
\end{lemma} 

\begin{proof}
By construction

\[
\phi_k (h) = \ln \left(\left\|h-P_d \right\|^2 \right) - \frac{1}{k} \sum_{i=1}^{M} \ln \left(\left\|h-P_i \right\|^2 \right). 
\]
Hence \[\lim \limits_{\left\| h \right\| \rightarrow \infty } \phi_k (h) = \ln \left( \lim \limits_{\left\| h \right\| \rightarrow \infty }  \left\|h\right\|^{2(k-M)}  \right) = + \infty\] for $k > M$.	
	\end{proof}

Let us next verify that $\phi_k$ has well behaved critical points in~$\mathcal{P}$. The following result concerns special obstacle arrangements for which $\phi_k$~has~\emph{degenerate}~\mbox{critical}~points. 
 
 \begin{proposition}
	\label{prop:Morse}
For every choice of $k\ge 0 $, there exist workspace arrangements for which    $\phi_k$ is non-Morse in~$\mathcal{P}$.
\end{proposition}
\begin{proof}
A smooth function is called a {\it Morse} function iff all its critical points are non-degenerate. Degenerate critical points, are critical points where the Hessian has zero eigenvalue(s). Let $\mathcal{C} (\phi_k)$ denote the set of critical points of $\phi_k$. We will denote with $h_d = h - q_d$ and $h_i=h-P_i$. Taking the gradient of $\phi_k$ we have:
	\begin{equation}
	\label{eq:gradient}
	\nabla \phi_k (h) = 2\frac{h_d}{\left\|h_d\right\|^2} - \frac{2}{k}\sum \limits_{i=1}^M \frac{h_i}{\left\|h_i\right\|^2} 
	\end{equation}
	Now since the Hessian is a linear operator, the Hessian of $\phi_k$ can be written as the superposition of the Hessians of $\phi_i$ and $\phi_d$:
	\[ \mathcal{H}_{\phi_k} (h) = \mathcal{H}_{\phi_d} (h) +\sum \limits_{i=1}^M \mathcal{H}_{\phi_i} (h)\] 
		Using eigendecomposition on the  Hessians of $\phi_d $ and $\phi_i$, the Hessian of $\phi_k$ assumes the following form: 
	\begin{equation}
	\label{eq:hessquadratic}
	\mathcal{H}_{\phi_k} = Q_d\cdot \Lambda_d\cdot Q_d^T -\frac{1}{k}\sum \limits_{i=1}^M Q_i \cdot \Lambda_i \cdot Q_i^T 
	\end{equation}
	where the eigenvalues of the Hessians of $\phi_d$ and $\phi_i$ form:  \[Q_x  = \left[ \hat{h}_x \; \vdots \; J\cdot \hat{h}_x \right], \quad x \in \left\{d,i\right\}  \]  with $J = \left[ \begin{array}{cc} 0 & -1 \\ 1 & 0 \end{array} \right]$. Observe that $Q_x^{-1}  = Q_x^T $. Now the eigenvalues of the Hessians of $\phi_d$ and $\phi_i$ form the following diagonal matrix:
	\[\Lambda_x = \frac{2}{\left\|h_x\right\|^2} L \] where $L=\begin{bmatrix} -1 & 0 \\0 &1 \end{bmatrix}$. Since $\phi_k$ is a harmonic function it will have two eigenvalues of equal magnitude and opposite signs. To establish that $\mathcal{H}_{\phi_k}$ is degenerate we need to show that there is a test direction $v$ and its perpendicular $J\cdot v$,  such that: 
	\[ v^T \cdot \mathcal{H}_{\phi_k} \cdot v = v^T \cdot J^T\cdot \mathcal{H}_{\phi_k} \cdot J\cdot v = 0.\] We will use as a test direction $ v = \hat h_d$. Then:

	\[\hat h_d^T\cdot \mathcal{H}_{\phi_k}\cdot \hat h_d = 	\hat h_d^T\cdot Q_d\cdot \Lambda_d\cdot Q_d^T\cdot \hat h_d^T -\frac{1}{k}\sum \limits_{i=1}^M \hat h_d^T\cdot Q_i \cdot \Lambda_i \cdot Q_i^T \cdot \hat h_d^T\]
expanding:
\[\hat h_d^T\cdot \mathcal{H}_{\phi_k}\cdot \hat h_d = 	-\frac{2}{\left\|  h_d \right\|^2} - \frac{1}{k}\sum \limits_{i=1}^M \left[ c_{id} \; \vdots \; s_{id} \right] \cdot \Lambda_i \left[ c_{id} \; \vdots \; s_{id} \right]^T  \]	
	where: 	
	 \[c_{id} = \cos(\theta_{id}) = \hat{h}_d^T \cdot \hat{h}_i  \] and 
	 \[s_{id} = \sin(\theta_{id}) =  \hat{h}_d^T \cdot J \cdot \hat{h}_i  =-\hat{h}_d^T \cdot J^T  \cdot \hat{h}_i \] where $\theta_{id}$ is the angle between $h_i$ and $\hat h_d$.
	
	We can now write: 
	\[\hat h_d^T\cdot \mathcal{H}_{\phi_k}\cdot \hat h_d = -\frac{2}{k} \lambda\]
and following similar analysis for the perpendicular test direction:
		\[ \hat h_d^T\cdot J^T\cdot \mathcal{H}_{\phi_k}\cdot J \cdot \hat h_d = \frac{2}{k} \lambda.\]
		where:
\begin{equation}
\label{eq:lambda}
\lambda = \frac{k}{\left\|  h_d \right\|^2} - \sum \limits_{i=1}^M \frac{c_{id}^2}{\left\| h_i \right\|^2} + \sum \limits_{i=1}^M \frac{ s_{id}^2}{\left\|h_i \right\|^2}.
\end{equation}
So the problem now becomes one of identifying at least one example of a symmetry that can cause $\lambda$ to vanish. One such example is the following: Assume $M=2\mu+1$ obstacles where $\mu \in \mathbb{Z}_+$, placed axisymmetrically and $h$ is placed such that $\theta_{id}=135^\circ$, $i\in \left\{2,\ldots \mu+1\right\}$, $\theta_{jd}=-135^\circ$, $j\in \left\{\mu+2,\ldots 2\mu+1\right\}$  and $\theta_{1d}=0^\circ$. Then the sum $- \sum \limits_{i=2}^M \frac{c_{id}^2}{\left\| h_i \right\|^2} + \sum \limits_{i=2}^M \frac{ s_{id}^2}{\left\|h_i \right\|^2} = 0$ and \[\lambda = \frac{k}{\left\|  h_d \right\|^2} -  \frac{1}{\left\| h_1 \right\|^2} \]
Requiring this to be zero we get: 
\begin{equation}
\label{eq:nondegreqlmd} 
\frac{1}{\left\|  h_d \right\|^2} =  \frac{1}{k\left\| h_1 \right\|^2}
\end{equation}

Now that we have identified a symmetry, we need to investigate whether this symmetry can render the point $h$ under investigation to be a critical point. To  this end, we will project $\nabla \phi_k (h)$ across two perpendicular directions, $\hat h_d$ and $\hat h_d^\bot = J\cdot \hat h_d$ and investigate whether the projection is zero for both. Define the sets $R= \left\{2,\ldots \mu+1\right\} $ and $L=\left\{\mu+2,\ldots M\right\}$. We have that:
\[\hat h_d^T \cdot J^T \cdot \nabla \phi_k(h) =     \frac{2}{k\sqrt{2}}\left( \sum \limits_{i\in R} \frac{1}{\left\|h_i\right\|} - \sum \limits_{i\in L} \frac{1}{\left\|h_i\right\|} \right) \]

\noindent where we have substituted $s_{id} = -s_{jd}$ for $i\in R$ and $j\in L$. 
Since $h_i$ and $h_j$ for $i\in R$ and $j\in L$  are symmetrically placed the  sums in the above equation will cancel each other and $\hat h_d^T \cdot J^T \cdot \nabla \phi_k(h) = 0$.
Now taking the projection across $h_d$, we have: 
\begin{equation}
\label{eq:critpreq}    
\hat h_d^T  \cdot \nabla \phi_k (h) = \frac{2}{k} \left( 2\sqrt{k} -  \sum \limits_{i=2}^M \frac{1}{\sqrt{2}}\frac{1}{\left\|h_i\right\|} \right) 
\end{equation}
where we have made use of Eq. (\ref{eq:nondegreqlmd}) and that $c_{id} = -\frac{1}{\sqrt{2}}$ for $i\in\left\{2,\ldots M \right\}$. As can be seen from Eq. (\ref{eq:critpreq}) for every choice of $k$ there is an appropriate choice of obstacle positions in the symmetry such that $\hat h_d^T  \cdot \nabla \phi_k(h):=0$. Hence $h\in \mathcal{C} (\phi_k) $ and the selected configuration will be a degenerate critical point. 
\end{proof}

This tells us that there is no universal $k$ that depends only upon  abstract information about the workspace such as the number of obstacles used in Lemma \ref{lem:attactive} that would enable us to establish 
on-the-fly that for a workspace with a given number of obstacles, $\phi_k$ will be Morse. Hence, ensuring on-the-fly that $\phi_k$ has the Morse property 
seems unrealistic. However,
the following proposition and two lemmas establish the "nice" behaviour of $\phi_k$ in~$\mathcal{P}$.

\begin{proposition}
\label{prop:isolated}
The critical points of $\phi_k$ are isolated
\end{proposition}
\begin{proof}
Since the logarithmic function is an analytic function, the sum of logarithmic functions will also be analytic. Hence $\phi_k$ is an analytic function. This implies that its Taylor expansion converges on any open set of its domain. Let $E$ be  an open set containing a critical point $h_c$ of $\phi_k$. Assume that the critical point is not isolated. This implies that there is either (a) an open neighborhood $Z$ around $h_c$ where $\forall h_z \in Z: \phi_k(h_z)=\phi_k(h_c)$ or (b) there is a 1-D manifold $\mathcal{L}$, where $\forall h_L \in \mathcal{L}: \phi_k(h_L)=\phi_k(h_c)$.

Let us now examine case (a) where "there is an open neighborhood $Z$ around $h_c$ where $\forall h_z \in Z: \phi_k(h_z)=\phi_k(h_c)$~". Then
the Taylor expansion of $\phi_k$ around $h_c$ will be:
\[\phi_k(h) = \phi_k(h_c) +\sum \limits_{n=1}^{\infty} \left\{ \frac{1}{n!}\sum \limits_{m=0}^n\left( \begin{array}{c}n \\ m \end{array} \right) D_{n,m}\left. \right|_{h_c} h_{cx}^{n-m}h_{cy}^{m}\right\}\]

where $D_{n,m}=\frac{\partial ^n \phi_k}{\partial x^{n-m} \partial y^m}$ and $h-h_c = \left[  h_{cx} \;\; h_{cy} \right]^T$. Since $\phi_k(h_z)=\phi_k(h_c), \forall h_z \in Z $, this necessarily implies that $\forall n>0, m\ge 0 : D_{n,m}=0$. Since $\phi_k$ is analytic, for any open connected component $G$ of the domain of $\phi_k$ containing $Z$ as a subset, it will hold $\phi_k(h_g) = \phi_k(h_c), \forall h_g \in G$, which according to the  Principle  of Permanence  implies that $\phi_k$ will be identically constant. However this is not true and by contradiction case (a) is not possible.

Let us now proceed to examining case (b) where "there is a 1-D manifold $\mathcal{L}$, where $\forall h_L \in \mathcal{L}: \phi_k(h_L)=\phi_k(h_c)$". 
We can now construct an analytic function $\lambda: \mathbb{R}_{\ge 0} \rightarrow \mathcal{L} $ with  $\lambda(0)=h_c$ and such that $\lambda(\ell)=h_{\ell}$ where $\ell$ is the length of the curve from $h_c$ to $h_{\ell}$. Define the function $\phi^L (\cdot) \triangleq (\phi_k \circ \lambda) (\cdot)$. Since  $\phi^L:\mathbb{R}_{\ge 0} \rightarrow \mathbb{R} $ is a composition of analytic functions, it will also be analytic. However by construction $\phi^L (\ell) = \phi_k(h_c)$, and this will be true for its whole domain $\ell \in \mathbb{R}_{\ge 0}$ according to the Principle of Permanence. This implies that $\mathcal{L}$ will have an infinite length. Then $\mathcal{L}$ will have infinitely large elements. Pick one of them, lets say  $\ell_i \in \mathcal{L}$ such that $\left\| \ell_i \right\| \rightarrow \infty$.  Then we have that $\phi_k(\ell_i) = \phi_k(h_c)$. However from Lemma \ref{lem:attactive} we have that $\phi_k(\ell_i)\rightarrow \infty$ which is a contradiction since for a valid workspace $\phi_k(h_c)$ is finite. Hence the assumed case (b) is not possible. 

Since cases (a) and (b) were shown to not be possible, then by contradiction  the assumption that the  critical point is not isolated is false and consequently all the critical points of $\phi_k$ are isolated.
\end{proof}

Degenerate critical points are atypical for analytical  functions in the sense that the set of Morse functions is dense and any non-Morse smooth function can be approximated by a Morse function \cite{MorseFloer}. Nevertheless the nature of a critical point is not clear when the Hessian becomes degenerate (in our case its rank may become zero). The following establishes the nature of the critical points.

\begin{lemma}
\label{lem:saddlep}
The  critical points of $\phi_k$ are saddle points.
\end{lemma}
\begin{proof}
For the case of non-degenerate critical points this is trivial since from the proof of Lemma \ref{lemma:locminfree} we have that $\lambda_1=-\lambda_2 \neq 0$ at a critical points.
For the case of degenerate critical points, let $h_d \in \mathcal{C}(\phi_k)$ be a degenerate critical point and $\theta_d = \phi_k(h_d)$. Then by the Mean Value Property of Harmonic Functions \cite{HFT2000} we have that:
\[ \theta_d = \int \limits_{S} \phi_k(h_d + r \zeta) d\sigma(\zeta)  \]
where $S$ is the boundary of a ball centered at $\theta_d$ with radius $r$ and $\sigma$ the normalized perimeter measure. Since $h_d$ is isolated, $\phi_k$ on $S$ cannot be identically $\theta_d$. In particular for the mean value on $S$ to be $\theta_d$, $\phi_k$ needs to take both higher and lower values than $\theta_d$ which implies that $h_d$ is a saddle point.
\end{proof}

The next Lemma establishes the nature of the basins of attraction of the critical points:

\begin{lemma}
\label{lem:attbasin}
The basin of attraction of the saddle points of $\phi_k$
is a set of measure zero.
\end{lemma}
\begin{proof}
Since $\Phi(\cdot)$ is a diffeomorphism, it is sufficient to prove this for $\phi_k$. 
Assume a basin of attraction $\Omega$ of the critical point $h_c \in \mathcal{C}(\phi) $ such that $\mu(\Omega)>0$ and $h_c \in \partial \Omega$, where $\mu(\cdot)$ is the Lebesque measure of a set. Then $\Omega$ is a positively invariant set with respect to the negated gradient flows of $\phi_k$, which will be pointing inward along $\partial \Omega$, i.e. $\nabla \phi_k(h_\omega) \cdot \hat n \ge 0$, $\forall h_\omega \in \partial \Omega$, where $\hat n$ the outward pointing normal to set's boundary. Take $\Omega_n\subseteq \Omega$ to be the largest subset of $\Omega$ that contains $h_c$ and for which $\nabla \phi_k(h_\omega) \cdot \hat n > 0$, $\forall h_\omega \in \partial \Omega_n \setminus h_c$. Such a set is guaranteed to exist as long as $\phi_k$ is not flat in $\Omega$. Hence 
\begin{equation}
\label{eq:lineint}
\oint_{\partial \Omega_n} \nabla \phi_k \cdot \hat n \,dl > 0
\end{equation}
Using the 2-dimensional Divergence Theorem we have that: 
\begin{equation}
\label{eq:divthm}
\oint_{\partial \Omega_n} \nabla \phi_k \cdot \hat n \,dl = \int_{\Omega_n} \nabla \cdot \nabla \phi_k \, dA 
\end{equation}
where $dA$ the area differential. Since $\phi_k$ is a Harmonic function, then by definition $\nabla \cdot \nabla \phi_k = \Delta \phi_k = 0 $ and consequently $\oint_{\partial \Omega_n} \nabla \phi_k \cdot \hat n \,dl = 0$. This is in contradiction with eq. (11) and since $\phi_k$ in $\Omega$ is non-flat,  $\Omega$ reduces to a set of measure zero. 
\end{proof}

\begin{theorem}
\label{prop:isNF}
Given a smooth navigation transformation $\Phi(\cdot)$, the function $\varphi_k(x) = \sigma \circ \phi_k \circ \Phi (x)$ is a navigation function on $\mathcal{W}$ as per Definition \ref{def:NFDefinition} for $k > M$.
\end{theorem}

\begin{proof}  Let us examine the properties of Definition \ref{def:NFDefinition}.
\subsubsection*{Property 1}
The function $\varphi_k(\cdot )$ is smooth since it is a composition of smooth functions.
\subsubsection*{Property 2}
From Lemma \ref{lemma:locminfree} we have that $\phi_k$ is free of local minima. From (\ref{eq:NFcand}), using the chain rule, we have that $\nabla \varphi_k = (\frac{d\sigma}{d\phi} \circ \phi_k \circ \Phi) J^T_\Phi (\nabla \phi \circ \Phi)$, where $J_\Phi$ is the Jacobian of $\Phi$. Since $\Phi$ is a diffeomorphism, $J_\Phi$ is non-singular and since $\sigma$ is strictly increasing, away from the destination $x_d$ the set of critical points in $\mathcal{W}$ and the corresponding point world match,  $\mathcal{C}(\phi_k)  = \Phi(\mathcal{C}(\varphi_k) - \{x_d\}) $. Also at the critical points away from $x_d$ we have that: $\nabla^2 \varphi_{k_{\left| \mathcal{C}_{\phi_k} \right.}} = (\frac{d\sigma}{d\phi} \circ \phi_k \circ \Phi) J^T_\Phi (\nabla^2 \phi \circ \Phi) J_\Phi$. Hence, using the same arguments as in \cite{KodRimNF90},
away from $x_d$ we have $index(\varphi)_ {\left| \mathcal{C}_{\varphi_k} \right.} = index(\phi)_ {\left| \mathcal{C}_{\phi_k} \right.}$. Since $\phi_k$ has no local minima
in $\mathcal{P}$, the same will be true for $\varphi$ in $\mathcal{W}$. Now let us examine $x_d$. 
From Eq. (\ref{eq:NFcand}) and using the notation $h = \Phi(x)$ we have that:

\[ 
\varphi_k(h) = \frac{\left\|h-P_d\right\|^2}{\left\|h-P_d\right\|^2 + \prod \limits _ {i=1} ^{M} \left\|h-P_i\right\|^{\frac{2}{k}} }
\]

\noindent Hence $\nabla \varphi_k(h) =\frac{2(h-P_d)D -\left\|h-P_d\right\|^2 \nabla D}{D^2}$ where $D = \left\|h-P_d\right\|^2 + \prod \limits_{i=1} ^{M} \left\|h-P_i\right\|^{\frac{2}{k}}$. The Hessian of $\varphi_k$ at $x_d$ is thus

\begin{equation}\label{eq:destphik2} 
\mathcal{H}_{\varphi_k}(x_d) =   \frac{2}{\prod \limits _ {i=1} ^{M} \left\|P_d-P_i\right\|^{\frac{4}{k}} }  I .
\end{equation}

\noindent The destination $x_d$ is thus the unique minimum of $\varphi_k$ in $\mathcal{W}$.
\subsubsection*{Property 3}
By Proposition \ref{prop:isolated} the critical points of $\phi_k$ are isolated. As discussed in the previous,  $\mathcal{C}(\phi_k)  = \Phi(\mathcal{C}(\varphi_k) - \{x_d\} )$. Hence $\varphi_k$ has no additional critical points except $x_d$, which was shown 
to be a non-degenerate global minimum. Hence all critical points of $\varphi_k$ are isolated. By Lemma~\ref{lem:saddlep} the isolated critical points except $x_d$ are  saddle points and by Lemma~\ref{lem:attbasin} the basin of attraction of these critical points has zero measure. Since $\Phi(\cdot)$ is a diffeomorphism, the above properties will also hold for $\varphi_k(x)$.
\subsubsection*{Property 4}
By construction $\lim \limits_{h\rightarrow P_i} \phi_k(h) \!=\! \infty$~for~$i = 1\ldots M$. Moreover, we have the requirement that $\lim\limits_{x \rightarrow \mathcal{O}_0}   \left\| \Phi (x) \right\| = +\infty$ which in combination with Lemma \ref{lem:attactive} implies that 

\begin{equation}\label{eq:inftypP} \lim \limits _{x\rightarrow \partial \mathcal{W}} \phi_k \circ \Phi (x) = +\infty.
\end{equation} 

\noindent From Eq. (\ref{eq:switch}) we have that 

\begin{equation}
\label{eq:s1}
\lim \limits _{x\rightarrow \infty} \sigma(x) = 1 .
\end{equation}

\noindent Combining eqs. \eqref{eq:inftypP} and \eqref{eq:s1}) we obtain 
\[ 
\lim \limits _{x\rightarrow \partial \mathcal{W}} \varphi = 1
\]

\noindent which implies the admissibility property.
\end{proof}

\section{Navigation Function Controller Design}
\label{sec:CtrlDesign}

\subsection{Setup}
In the previous section we have shown that $\varphi_{k}$ where $k \ge M+1$ is a Navigation Function where $M$ is the number of workspace obstacles. As the robot moves in the environment, the number $M$ of obstacles encountered is increasing. Let $d$ be the detection radius of the robot. Any obstacle entering this radius is registered as a new obstacle and the number of obstacles $M$ increases when fist detected. Let $T = \left\{T_1,\ldots ,T_N \right\}$ be the time instants where new obstacles are discovered during the navigation of the robot and assume that $N$ is the total number of obstacles that are discovered during a navigation instance. Let $n(t)$ denote the number of obstacles that were discovered up to the time instant $t$ and define: 
\begin{equation}
    \label{eq:phiseq}
 \Theta(x,t) \triangleq  
\varphi_{k(t)}(x)  
\end{equation}
where $k(t)>n(t)$ an integer\footnote{It is sufficient to set $k(t)=n(t)+1$}.  Eq. (\ref{eq:phiseq}), is a continuous (smooth) function  in the state but discontinuous in time, representing in every time interval $T_{n(t)} \le t <T_{n(t)+1}$ the Navigation Function that is applied to the discovered environment at time $t$.

To ensure that the Navigation Transformation $\Phi$ is well defined at any time instant, we will have the following:
\begin{assumption}
\label{as:objknown}
Sensed  obstacles' shapes and orientations are known.
\end{assumption}

In addition we need to have an obstacle "memory" in our system as follows:
\begin{assumption}
\label{as:objpers}
Obstacles become part of the known workspace after being sensed.
\end{assumption}

Considering the definition of the symmetric sensing sector $S_{r,\theta}$, it can be shown that the minimum distance that a robot can approach to an obstacle before detecting it, is given by:  
\begin{equation}
\label{eq:dmin}    
d_{\min}(\theta) = \left\{ \begin{array}{ll}
 \min \left\{ r\sin\frac{\theta}{2},\frac{\rho_{\min}}{\cos\frac{\theta}{2}} \right\}, & \theta<\pi \\ 
 r, &  \pi \le \theta \le 2\pi
\end{array}
\right.
\end{equation}
where $\rho_{\min}$ is the minimum radius of curvature of the workspace obstacles. 

The following assumption is about knowledge of the starting neighborhood:
\begin{assumption}
\label{as:localobs}
For a symmetric sensing sector $S_{r,\theta}$, any obstacles that are within a radius of $d_{\min}(\theta)$  of the initial robot configuration, are assumed known.
\end{assumption}

\subsection{Kinematic Controller Design}
Let us now consider the first-order kinematic system model of Eq. (\ref{eq:holint}). Our motivation in setting up the control law for this system, is to  create a vector field that provides kinetic energy to the system that matches (or is proportional to) the level of the Navigation Function.  We have the following result:

\begin{proposition}
\label{prop:kinctrl}
System (\ref{eq:holint}) with a symmetric sensing sector $S_{r,\theta}$ under the control law: 
\begin{equation}
    \label{eq:kinctrl}
    u = - K\sqrt{2\Theta} \cdot \widehat{ \nabla \Theta}
\end{equation}
with $K>0$ a tuning parameter, is globally asymptotically stable almost everywhere as long as $r >  0 $ and $\theta >  0$
\end{proposition}
\begin{proof}\\
Case 1: Global Sensing, i.e. $r\rightarrow \infty$ and $\theta = 2\pi$:\\
For the case of global sensing in an environment with $M$ obstacles we have that   $n(t) = M, \forall t$  and $\Theta(x,t)=\varphi_k(x)$, with $k\ge M+1$, is a Navigation Function by construction. Choosing  $V = \Theta$ as a Lyapunov function candidate we have that: 
\[\dot V = -K \sqrt{2\Theta} \left\| \nabla \Theta(x,t) \right\| \le 0 \]
The set where $\dot V = 0$, consists only of the critical points and the destination configuration (global minimum). Since by Proposition \ref{prop:isNF}, function $\Theta$ is a Navigation Function, it has no local minima and its critical points are isolated saddle points with attractive basins that are sets of measure zero. Hence by LaSalle's Invariance Principle the system will converge to the largest invariant set that includes the destination configuration and the saddle points that are the $\omega$ limit set of a set of measure zero of initial conditions. Hence $\dot V \overset{a.e.}{<} 0 $ and the control law of Eq. (\ref{eq:kinctrl}) is globally asymptotically stable, almost everywhere.

Case 2: Local sensing i.e. $r$ is finite and $0< \theta \le 2\pi $:\\
In this case, the robot senses an obstacle's existence before the minimum distance to the obstacle becomes $d_{\min}(\theta)$. 

Let $t=T_{\nu}$ be the time instant that the robot discovered the $\nu$'th obstacle. At this time instant the control law (\ref{eq:kinctrl}) will have a discontinuity in time. This discontinuity will be an isolated event due to Assumption \ref{as:objpers}. Hence the trajectories of System (\ref{eq:holint}) under control law (\ref{eq:kinctrl}),  will flow along Carath\'eodory solutions, and following the same Lyapunov analysis as the previous case, we get: 

\[\underset{t}{\rm{ess}\,\sup } \left\{ -K \sqrt{2\Theta} \left\| \nabla \Theta(x,t) \right\| \right\} \le 0 \]
 
\noindent which implies $\dot V \overset{a.e.}{<} 0 $ and the control law of Eq. (\ref{eq:kinctrl}) is globally asymptotically stable, almost everywhere.
\end{proof}

\subsection{Dynamic Controller Design}
Now let us consider the case of the dynamical System (\ref{eq:dynmodel}). Let \begin{equation}
    \label{eq:totalenergy}
V(t) = \mu \Theta +\frac{1}{2} m \dot x^T \dot x
\end{equation}
be the total energy of the system. In this case our system will be converting potential energy to kinetic energy and vice-versa while dissipating energy through an appropriate dissipation term. We have the following result:
\begin{proposition}
\label{prop:dynctrlglobal}
System (\ref{eq:dynmodel}) with global sensing, 
under the control law: 
\begin{equation}
    \label{eq:dynctrl}
    f =  -\mu \nabla \Theta - \lambda \dot x       
\end{equation}
\noindent with $\mu, \lambda$ positive gains, is globally asymptotically stable almost everywhere.
 \end{proposition}
 \begin{proof}\\
Using the same reasoning as in the Case 1 of the  proof of Proposition \ref{prop:kinctrl} and choosing $V$  as a Lyapunov function candidate we have that: 
\[\dot V = \mu \nabla^T \Theta \cdot \dot x + \dot x^T \cdot (-\mu \nabla \Theta - \lambda \dot x) = -  \lambda \left\| \dot x \right\|^2 \le 0 \]
Invariant sets where $\dot V = 0$, consists only of the critical points and the destination configuration (global minimum). Since by Proposition \ref{prop:isNF}, function $\Theta$ is a Navigation Function, it has no local minima and its critical points are isolated saddle points with attractive basins that are sets of measure zero. Hence by LaSalle's Invariance Principle the system will converge to the largest invariant set that includes the destination configuration and the saddle points that are the $\omega$ limit set of a set of measure zero of initial conditions. Hence $\dot V \overset{a.e.}{<} 0 $ and the control law of Eq. (\ref{eq:dynctrl}) is globally asymptotically stable, almost everywhere.
\end{proof}
 
 The following result enable us to choose the minimum damping to avoid oscillations at the destination configuration.
 
 \begin{corollary}
 \label{cor:critdamp}
 Control law (\ref{eq:dynctrl}) with 
    \begin{equation}
     \label{eq:critdamp}
 \lambda = \lambda_c = 2\sqrt{2\mu m}\prod \limits _ {i=1} ^{M} \left\|P_d-P_i\right\|^{\frac{-1}{k}} 
  \end{equation} 
 provides critical damping at the neighborhood of the destination configuration.
 \end{corollary}
 \begin{proof}
 From Eq. (\ref{eq:destphik2}), the Navigation Function $\mu \Theta$ will behave locally at the destination neighborhood as an elastic potential, i.e. 
 $\mu \Theta \approx \frac{1}{2} k_{sp} x^Tx$, with spring constant: 
  \[k_{sp} = 2\mu\prod \limits _ {i=1} ^{M} \left\|P_d-P_i\right\|^{\frac{-2}{k}}  
  \]
 Since $\lambda$ acts as the damping coefficient for system (\ref{eq:dynmodel}) under control law (\ref{eq:dynctrl}), the damping ratio of the second order system will be: 
 \[\zeta = \frac{\lambda}{2\sqrt{m k_{sp}}}.
  \]
Setting $\zeta = 1$ for critical damping, we get the result in Eq. (\ref{eq:critdamp})
 \end{proof}

We can now state the following result in the case of local sensing:
\begin{proposition}
\label{prop:dynctrllocal}
System (\ref{eq:dynmodel}), with a symmetric sensing sector $S_{r,\theta}$ under the control law: 
\begin{equation}
    \label{eq:dynctrllocal}
    f =  -\mu \nabla \Theta - \lambda(t) \dot x       
\end{equation}
with
  \begin{equation}
\label{eq:dissipationfun}
\lambda (t) = \left\{ \begin{array}{ll}
 2\sqrt{2\mu m}\prod \limits _ {i=1} ^{n(t)} \left\|P_d-P_i\right\|^{\frac{-1}{k}} , & V < \mu \\ 
 \frac{m}{d_{\min(\theta)}}\left\| \dot x(T_{n(t)}) \right\| , & V \ge \mu
\end{array}
\right.
 \end{equation}
 and $\mu>0, r >0$ and $0< \theta \le 2\pi$ ensures collision avoidance everywhere, critical damping at the neighborhood of the destination configuration  and, almost everywhere global asymptotic stability. 
\end{proposition}
\begin{proof}
System (\ref{eq:dynmodel}) under control law (\ref{eq:dynctrllocal}) can be considered as a mechanical system with dissipation under a conservative force field. This implies that if the system starts from rest, then collisions are not possible as long as the total energy of the system is bounded below $\mu$. This is a direct result of the admissibility Property of the Navigation Function. 

However, it is possible that the jump in the potential level of $\Theta$, right after a new obstacle is encountered, can cause the total energy to go above $\mu$. In this case we need to ensure that, for the worst case scenario, there is sufficient dissipation in the system to absorb the energy added due to introduction of a new obstacle. 


Now assume that the robot is  heading towards the newly discovered obstacle. To simplify our analysis let us not take into account the repulsive force generated by the navigation function and assume that the robot is moving towards the obstacle with only the dissipation force slowing it down. Then the differential equation describing the motion of the robot is given by $m\ddot x +\lambda \dot x = 0$ where $x$ is the distance across the straight line connecting the robot position $x(T_{n(t)})$ when the obstacle was detected and the closest obstacle point. The solution to the differential equation is 
\begin{equation}
\label{eq:solpos}    
x(t) =\| \dot x(T_{n(t)})\| \frac{m}{\lambda}\left( 1-\exp ^{-\frac{\lambda}{m} t} \right)
\end{equation}
and its time derivative is 
\begin{equation}
\label{eq:solvel}    
\dot x(t) =\| \dot x(T_{n(t)})\| \exp ^{-\frac{\lambda}{m} t} 
\end{equation}
Combining equations (\ref{eq:solpos}) and (\ref{eq:solvel}) we get 
\[ x(t) =\frac{m}{\lambda}\left( \dot x(T_{n(t)}) - \dot x(t) \right)\]
Setting the final velocity to zero and requiring that the travelled distance is $d_{\min(\theta)}$, we get requirement on the damping coefficient: 
$ \lambda (t) = 
 \frac{m}{d_{\min(\theta)}}\left\| \dot x(T_{n(t)}) \right\| 
 $, that ensures that there is no collision due to the potential jump caused by new obstacle discovery.

Since there are finite obstacles in the workspace, and since no collision is possible, $\dot V(t) = - \lambda(t) \|\dot x \|^2$  for every time interval $T_{n(t)} \le t <T_{n(t)+1}$, and following the same reasoning as in the proof of Proposition \ref{prop:dynctrlglobal}, we have that $\dot V(t) \overset{a.e.}{<} 0 $ in every time interval. Now after the last workspace obstacle is discovered, the proof follows as  the proof of Proposition \ref{prop:dynctrlglobal} and we recover global asymptotic stability, almost everywhere.
\end{proof}

The following result is useful for determining the mechanical limits of the robot actuation system:

\begin{corollary}
\label{cor:boundedvel}
For system (\ref{eq:dynmodel}) under control laws (\ref{eq:dynctrl}) or (\ref{eq:dynctrllocal}), it holds that: 
\[\left\| \dot x \right\| < \sqrt{\frac{2\mu}{m}} \]
\end{corollary}
\begin{proof}
Since system (\ref{eq:dynmodel}) under control laws (\ref{eq:dynctrl}) or (\ref{eq:dynctrllocal}) represents a mechanical system with dissipation under a conservative force field the maximum velocity will occur at the lower potential level in the case that there is no dissipation and for a system that has started at the highest potential level. Due to the conservation of energy and from Eq. (\ref{eq:totalenergy}) we get the result. Please note that this represents a conservative bound.
\end{proof}

\section{Simulation Results}
\label{sec:sims}
The effectiveness of the methodology was verified through a set of non-trivial simulations. A symmetric sensing sector with $\theta=60^o$ and $r=1m$ was used. The Navigation Transformation was built based on \cite{NavTranTRO17} and \cite{ConLoizou20}. 

In the first simulation study, the kinematic controller (\ref{eq:kinctrl}) was applied to the system (\ref{eq:holint}). The robot started at its initial position in a "seemingly" empty workspace. As the robot moved, new obstacles were discovered through the  symmetric sensing sector as shown in Fig. \ref{fig:KinSim}, starting from $O_1$ up to $O_6$. As can be seen the proposed controller performed successfully, maintaining the robot in the workspace, handling the occurrence of new obstacles, avoiding collisions and stabilizing at the destination configuration.  
\begin{figure}[htp]
    \centering
    \includegraphics[trim=4cm 9cm 1cm 8.3cm, clip, scale=.61]{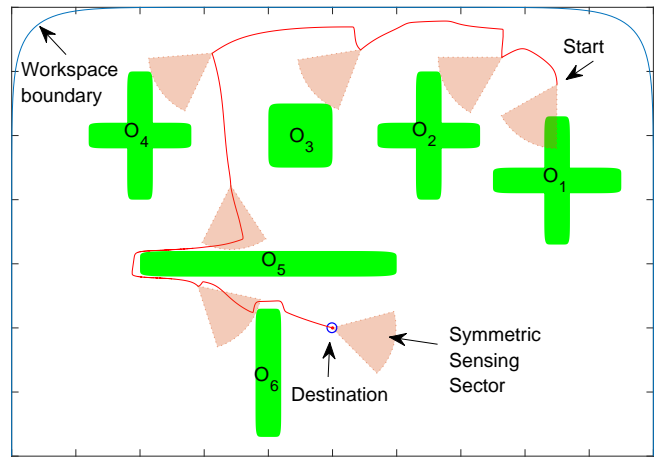}
    \caption{[Kinematic Controller] Trajectory of system (\ref{eq:holint}) under control law (\ref{eq:kinctrl}). Symmetric sensing sector shown at destination and every instance a new obstacle is discovered.}
    \label{fig:KinSim}
\end{figure}

In the second simulation study, the dynamic controller (\ref{eq:dynctrllocal}) was applied to the system (\ref{eq:dynmodel}). The mass of the system was set as $m=1Kg$. The controller parameters were chosen as $mu=10$, $\rho_{\min}=0.05m$ whereas the dissipation parameter $\lambda(t)$ was chosen as per Eq. (\ref{eq:dissipationfun}).
The robot started at the same position as in the first simulation study with the same workspace. The robot moved along the trajectory as shown in Fig. \ref{fig:DynSim}, sequentially discovering obstacles $O_1$ up to $O_6$. The robot trajectory in this case follows a similar - but not identical - path as the kinematic system. We can observe limited oscillatory behavior as the robot approaches $O_5$ and $O_6$. However, as predicted by our analysis, the robot demonstrated critical damping without oscillations at the destination configuration. The maximum velocity of the robot in this case study was $\left\| \dot x \right\|_{\max} = 0.93m/s$, much lower than the upper bound of $4.47m/s$ provided by Corollary \ref{cor:boundedvel}. As can be seen the proposed controller performed successfully, maintaining the robot in the workspace, handling the appearance of new obstacles, avoiding collisions and stabilizing without oscillations at the destination configuration.  
\begin{figure}[htp]
    \centering
    \includegraphics[trim=4cm 9cm 1cm 8.3cm, clip, scale=.61]{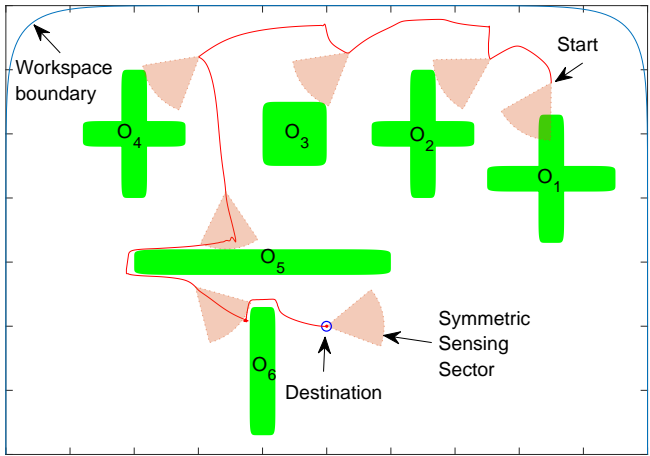}
    \caption{[Dynamic Controller] Trajectory of system (\ref{eq:dynmodel}) under control law (\ref{eq:dynctrllocal}). Symmetric sensing sector shown at destination and every instance a new obstacle is discovered.}
    \label{fig:DynSim}
\end{figure}

\section{Conclusions}
\label{sec:conclude}
This paper presents a correct-by-construction methodology to create Navigation Functions, without the need for tuning. The methodology builds on the concepts of Navigation Transformation \cite{NavTranTRO17}, Harmonic Function Based Navigation Functions \cite{LoizouCDC11} and of course the Navigation Functions \cite{KodRimNF90}. Without the need for tuning, on-the-fly addition of new obstacles to the workspace is possible, while maintaining the convergence and stability properties. Considering issues from sensor based robot navigation, the concept of symmetric sensing sector is introduced, modelling typical robot sensors like the laser scanner. It is shown that navigation is possible for any non-zero symmetric sensing sector. A kinematic controller that creates a vector field with kinetic energy that matches the level of the Navigation Function is proposed, in order to avoid the slowdowns close to the saddle points. The proposed dynamic controller, in addition to the stability and collision avoidance guarantees, provides critical damping at the destination configuration and bounds on the maximum velocity of the system. In addition to the analytical guarantees, simulation results verify the performance of the system.

Future work includes extending the methodology to handle moving obstacles, obstacle and workspace mapping,  multi-agent scenarios and 3D navigation.

\addtolength{\textheight}{-12cm}   




\section*{ACKNOWLEDGMENT}
The first author would like to acknowledge the contribution of European
Union's Horizon 2020 research and innovation program
under grant agreements 824990 (RIMA), 767642 (L4MS)
and by the Cyprus Research and Innovation Foundation grants
EXCELLENCE/1216/0365
(HOD-ICCCS)
and
EXCELLENCE/1216/0296 (RETuNE).

 \bibliography{Bibliography}
\bibliographystyle{plain}

\end{document}